\documentclass{bmvc2k}

\addauthor{Sami S. Brandt}{www.itu.dk}{1}

\addauthor{Hanno Ackermann}{www.tnt.uni-hannover.de/staff/ackerman/}{2}

\addinstitution{
IT University of Copenhagen,\\ Copenhagen, Denmark\\
}

\addinstitution{
 Leibniz Universit\"at Hannover,\\ Hannover, Germany\\
}

\runninghead{Brandt, Ackermann}{NRSFM by Rank-One Basis Shapes}

\def\etal{\emph{et al}\bmvaOneDot}

\usepackage{times}
\usepackage{epsfig}
\usepackage{graphicx}
\usepackage{amsmath}
\usepackage{amssymb}
\usepackage{mathrsfs}
\usepackage{color}
\usepackage{algorithm}

\newtheorem{theorem}{Theorem}[section]
\newtheorem{lemma}[theorem]{Lemma}

\newenvironment{proof}{\noindent \emph{Proof.} }{\hfill $\Box$}

\newcommand{\db}{\mathbf{d}}
\newcommand{\G}{\mathbf{G}}
\newcommand{\D}{\mathbf{D}}
\newcommand{\V}{\mathbf{V}}
\newcommand{\U}{\mathbf{U}}
\newcommand{\Sm}{\mathbf{S}}
\newcommand{\W}{\mathbf{W}}
\newcommand{\x}{\mathbf{x}}
\newcommand{\m}{\mathbf{m}}
\newcommand{\tr}{\mathbf{t}}
\newcommand{\I}{\mathbf{I}}
\newcommand{\w}{\mathbf{w}}
\newcommand{\A}{\mathbf{A}}

\newcommand{\M}{\mathbf{M}}
\newcommand{\bb}{\mathbf{b}}
\newcommand{\bd}{\mathbf{d}}
\newcommand{\B}{\mathbf{B}}

\newcommand{\T}{\mathrm{T}}

\newcommand{\proj}{\mathscr{P}}

\begin{document}

\title{Non-Rigid Structure-From-Motion by Rank-One Basis Shapes}

\author{First Author\\
Institution1\\
Institution1 address\\
{\tt\small firstauthor@i1.org}
\and
Second Author\\
Institution2\\
First line of institution2 address\\
{\tt\small secondauthor@i2.org}
}

\maketitle

\begin{abstract}
    In this paper, we show that the affine, non-rigid structure-from-motion problem can be solved by rank-one, thus degenerate, basis shapes. 
    It is a natural reformulation of the classic low-rank method by Bregler \etal, 
    where it was assumed that the deformable 3D structure is generated by a linear combination of rigid basis shapes. The non-rigid shape will be decomposed into the mean shape and the degenerate shapes, constructed from the right singular vectors of the low-rank decomposition. The right singular vectors are affinely back-projected into the 3D space, and the affine back-projections 
    will also be solved as part of the factorisation. 
    By construction, a direct interpretation for the right singular vectors of the low-rank decomposition will also follow: they can be seen as principal components, hence, the first variant of our method is referred to as Rank-1-PCA.
    The second variant, referred to as Rank-1-ICA, additionally estimates the orthogonal transform which maps the deformation modes into as statistically independent modes as possible. It 
    has the advantage of pinpointing statistically dependent subspaces related to,
    for instance, lip movements on human faces. 
    Moreover, in contrast to prior works, no predefined dimensionality for the subspaces is imposed.
    The experiments on several datasets show that the method achieves better results than the state-of-the-art, it can be computed faster, and it provides an intuitive interpretation for the deformation modes. 
\end{abstract}

\section{Introduction}

Non-rigid structure-from-motion (NRSFM), the problem of reconstructing both the scene geometry and dynamic, deforming object structure, is a classic problem in computer vision. NRSFM in general is a difficult problem, although there have been significant developments in the last two decades. The starting point for NRSFM can be seen as the work by Bregler \etal~\cite{Bregler00}, who proposed a low-rank approach, where the underlying assumption is that the deformable 3D shape can be represented as a linear combination of rigid 3D basis shapes. This leads to a matrix factorisation problem that
can be seen as the generalisation of the classic Tomasi--Kanade~\cite{Tomasi92} factorisation. 
One crucial characteristic of the classic NRSFM problem is the fact that the decomposed motion matrix has a block-form structure due to the assumption of 3D basis shapes. It was found out that the general solution needs to tackle with the inherent geometric and structural ambiguities of the problem that has been challenging to date.

There have been numerous approaches to address the NRSFM problem. The majority of  previous works have assumed a calibrated affine camera 
and utilised the well-known orthogonality constraints for camera matrices. Additional constraints used include heuristic deformation minimisation \cite{Brand01}, constraints arising from stereo rig \cite{DelBue04}, shape basis fixation on certain frames \cite{Xiao06}, and factoring a multifocal tensor \cite{Hartley08}. Physical and temporal priors have also been widely used such as those for rigidity \cite{DelBue06,Bartoli08}, camera trajectory smoothness \cite{Gotardo11}, temporal smoothness \cite{Torresani08,Akhter09}, and deformation \cite{Brand01,DelBue12}.
The problem has alternatively been viewed as manifold learning \cite{DelBue12} that has naturally led to alternation-based optimisation \cite{Torresani01,Paladini12,Torresani08}. In addition, Bartoli \etal \cite{Bartoli08}  proposed a coarse-to-fine solution that uses  information on several scales to regularise the solution. There have also been uncalibrated approaches \cite{Brandt09,Brandt18} that assume statistical independence of the shape bases to solve the structural and geometric ambiguities. Dai \etal \cite{Dai12} completely ignored the structural ambiguity by using the observation that the reconstruction is not ambiguous unlike the shape basis.

\begin{figure}[tb]
    \centering
    \includegraphics[width=0.8\textwidth,trim={0cm 0.5cm 4cm 1.9cm},clip]{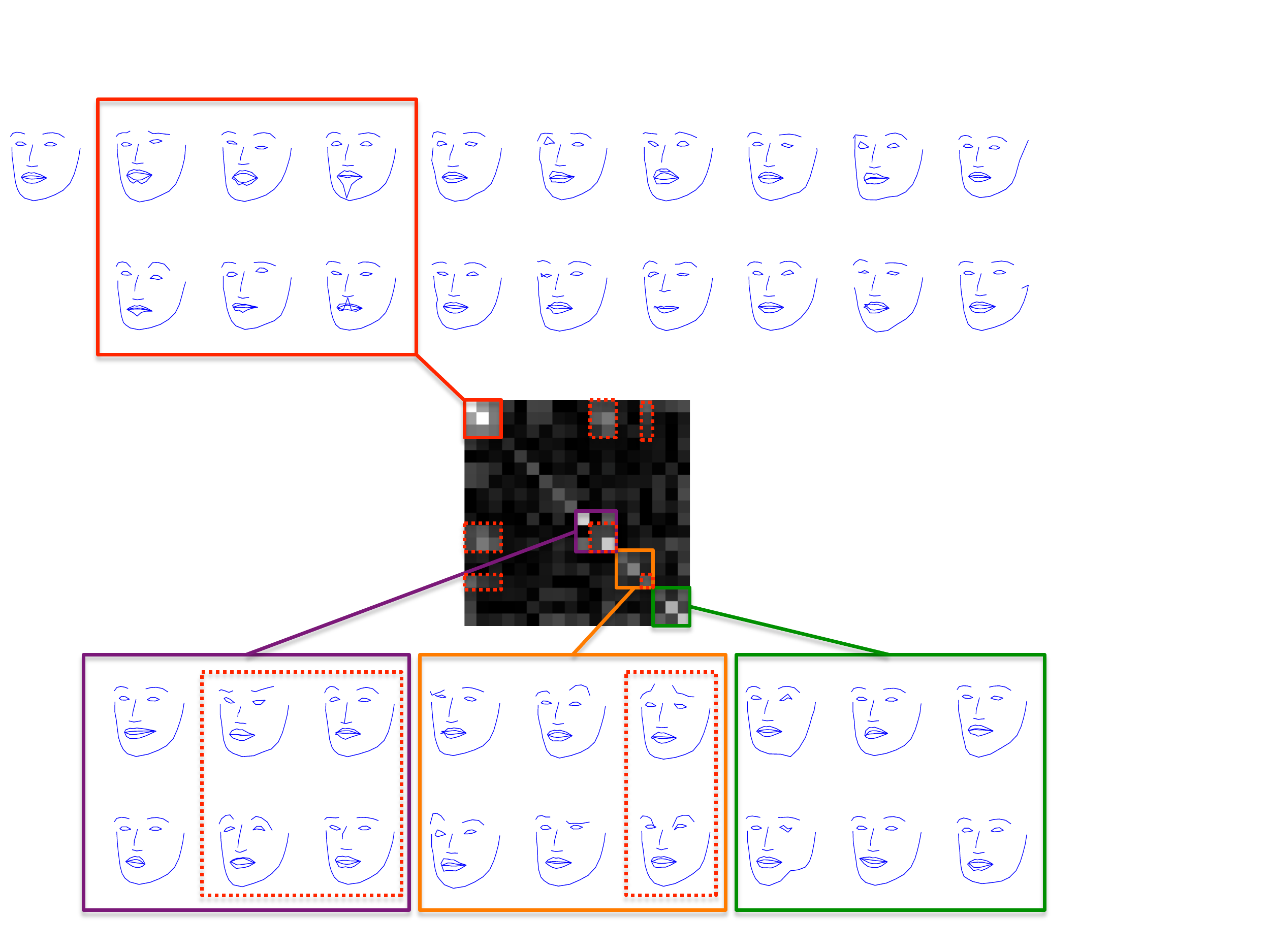}
    \caption{In this work, the basis shapes are rank-$1$ shapes $\B_k$, here illustrated as two sided perturbation $\B=\B_0 \pm \alpha \B_k$ from the mean rigid shape $\B_0$ (left), where the shape basis have been constructed to be as statistically independent as possible. The coefficient covariance matrix $\mathbf{C}_\alpha$ reveals the statistically dependent subspaces such as those related to lip movements (red), asymmetric eyebrow movements (orange), and face size deformations (green).}
    \label{fig:motivation}
\end{figure}

This work reformulates the non-rigid, low-rank model. 
Instead of trying to solve the harder problem of finding the underlying 3D shape basis, we individually analyse the singular vectors that form the shape matrix  and back-project them onto 3D to create degenerate, rank-$1$ shapes.  
In other words, the 3D shapes will be modelled as a decomposition, where the mean shape is perturbed by one-dimensional components, i.e., the shape approximation is updated in one direction at a time, instead of updating the shape with rank-$3$ shape components. Our approach has substantial advantages. First, it is a natural utilisation of the SVD since each singular vector is separately used, in the order of significance, to update the 3D shape approximation. Second, no block-form structure needs to be enforced as the structural property will be automatically addressed by the rank-$1$ shapes. Third, without a loss of generality, we can even orthogonally transform the $K$-dimensional singular space in order to cluster the rank-$1$ shapes into subspaces of arbitrary dimension smaller than $K$ so that the resulting subspaces are as independent as possible from each other. The components within a subspace are statistically dependent, as Fig.~\ref{fig:motivation} illustrates.   

The organisation of this paper is as follows. In Section~\ref{sec:degenerate}, we describe the mathematical model and in Section~\ref{sec:factorisation} the factorisation model is presented. In Section~\ref{sec:ICA}, we show how to make an orthogonal transform to the $K$-dimensional singular space that makes the rank-$1$ components as statistically independent as possible. The recovery of the rank-$1$ shape basis is described in Section~\ref{sec:recovery}. In Section \ref{sec:experiments}, we report our experiments and we conclude in Section~\ref{sec:conclusions}.

%
%

\section{Method} \label{sec:method}

\subsection{Degenerate Basis Shapes Model} \label{sec:degenerate}

In contrast to the standard factorisation model which is based on $3$-dimensional basis shapes, we assume that the basis shapes are degenerate, rank-$1$ shapes. In effect, the 3D shapes are represented as $\x^i_j=\bb_{0j} + \sum_{k=1}^K \alpha_k^i \bb_{kj}$, where $\x^i_j=\bb_{0j}$ is the rigid, mean shape, $\alpha_k^i$ is a scalar and $\mathrm{rank} (\bb_{k1}\ \bb_{k2}\ \cdots \bb_{kJ})=1$ for $k \neq 0$. 
The 2D projection $\hat{\m}_j^i$ of the 3D point $\x^i_j$ is thus
\begin{equation}
  \hat{\m}_j^i= \M^i \x^i_j  +{\tr}^i = \M^i \left(\bb_{0j} + \sum_{k=1}^K  \alpha_k^i \bb_{kj}  \right) +{\tr}^i,
\end{equation}
where $\M^i$ is a $2\times3$ projection matrix to the image $i$ and $\tr^i$ is the corresponding translation vector. 

The maximum likelihood solution with respect to the parameters ${\M^i,\tr^i,\alpha_k^i,\bb_{kj}}$, $i=1,\ldots,I$, $j=1,\ldots,J$, $k=1,\ldots,K$, 
with Gaussian noise model, minimises the squared loss  
%
\begin{equation}
  E=\sum_{i,j} \| \hat{\m}_j^i-\m_j^i \|^2 \equiv \| \W-\hat{\W} \|_\mathrm{Fro}^2.  \label{eq:fro}
\end{equation}
Here, the translation corrected measurements $\m_j^i-\hat{\tr}^i$, $\hat{\tr}^i=\frac{1}{J} \sum_j \mathbf{m}_j^i$, are collected into the matrix $\W$, 
so that
\begin{equation}
{\W}\simeq\underset{\triangleq \M}{\underbrace{
    \begin{pmatrix}
       \M^1  &  \alpha_1^1 \M^1  & \cdots &  \alpha^1_K \M^1 \\
       \M^2 &  \alpha_1^2 \M^2 & \cdots &  \alpha^2_K \M^2 \\
      \vdots          & \vdots          & \ddots &  \vdots\\
       \M^I &  \alpha_1^I \M^I & \cdots &  \alpha^I_K \M^I \\          
    \end{pmatrix} }}\underset{\triangleq \B}{\underbrace{
    \begin{pmatrix}
      \B_0\\
      \B_1\\
      \vdots\\
      \B_K
    \end{pmatrix}}}, \label{eq:factorisedform}
\end{equation}
where $\B_k=\left( \bb_{k1} \ \bb_{k2}\ \cdots \ \bb_{kJ} \right)$, $\B_0$ is the rigid shape, $\mathrm{rank}(\B_0)\leq 3$, and $\mathrm{rank}(\B_k)=1$, $k \neq 0$. 
Hence, the noise free measurement
matrix has the rank constraint $\mathrm{rank} (\hat{\W}) \leq K+3$.\footnote{In the classic low-rank factorisation model $\mathrm{rank} (\hat{\W}) \leq3K+3$.} Without a loss of generality, we additionally require that $\| \B_k \|_\mathrm{Fro}=1$ for all $k \neq 0$.

\subsection{Factorisation}
\label{sec:factorisation}

 %
 The best rigid affine reconstruction along with the inhomogenous camera matrices is obtained by the standard Tomasi-Kanade factorisation \cite{Tomasi92}. That is, we factorise the translation corrected measurement matrix $\W$ by singular value decomposition and truncate all the singular values, and singular vectors, up to the three largest 
\begin{equation}
    \W_0 = \M_0 \B_0. 
    \label{eq:fac.tomasi}
\end{equation}
The inhomogeneous projection matrices, up to an affine transform, 
are $\M_0=\frac{1}{\sqrt{J}}\U_0 \Sm_0$ and the mean rigid shape is $\mathbf{B}_0 
= \sqrt{J} \V_0^\T$. 
%
%
We then subtract the rigid component from the measurement matrix
\begin{equation}
    \Delta \W = \W - \W_0, \label{eq:nonrigid.part}
\end{equation}
and continue with the non-rigid part $\Delta \W$.

 
 
For the non-rigid part, the remaining constraint is $\mathrm{rank} (\Delta \W) = K$. We hence truncate all the singular values, and singular vectors, up to the $K$ largest 
\begin{equation} 
  \Delta \W \approx \Delta \tilde{\W} \equiv {\U}' {\Sm}' {\V}'^\T = {\M}' {\B}',\label{eq:svd}
\end{equation}
where ${\M}'=\frac{1}{\sqrt{J}}{\U}' {\Sm}'$ and  ${\mathbf{B}}'= \sqrt{J} {\V}'^\T$. The remaining problem is to find the $3K\times K$ operator $\A$ so that $\B=\A \B'$ and $\M=\M' \A^\dagger$ corresponding to (\ref{eq:factorisedform}). 
Since $\B'$ has $K$ linearly independent rows, $\A$ can be written in the form $\A=\D \G$, where $\D$ is block diagonal matrix with $3\times1$ blocks $\db_k$, and $\G$ is an orthogonal matrix.  

The selection of the orthogonal transformation $\G$ is an additional freedom arising from the rank-$1$ decomposition. Setting $\G$ into $\I$ principally corresponds to doing Principal Component Analysis with degenerate shapes (see Fig.~\ref{fig:faceModes}a,c). 
In the following, we will refer to this procedure as {Rank-$1$-PCA}. Although no grouping of the rank-one components is strictly necessary, we also consider an alternative way of estimating the rank-$1$ shapes by setting $\G$ so that they are as \emph{statistically independent factors} as possible. This will allow us to analyse statistically linked shape components, such as lip movements, by isolating them from the other deformations (Fig.~\ref{fig:faceModes}b,d).
We will refer to this procedure as {Rank-$1$-ICA.}

\begin{figure*}[tb]
\begin{center}
\subfigure[]{\includegraphics[width=0.47\textwidth, trim={4cm 3.8cm 5cm 0.5cm},clip]{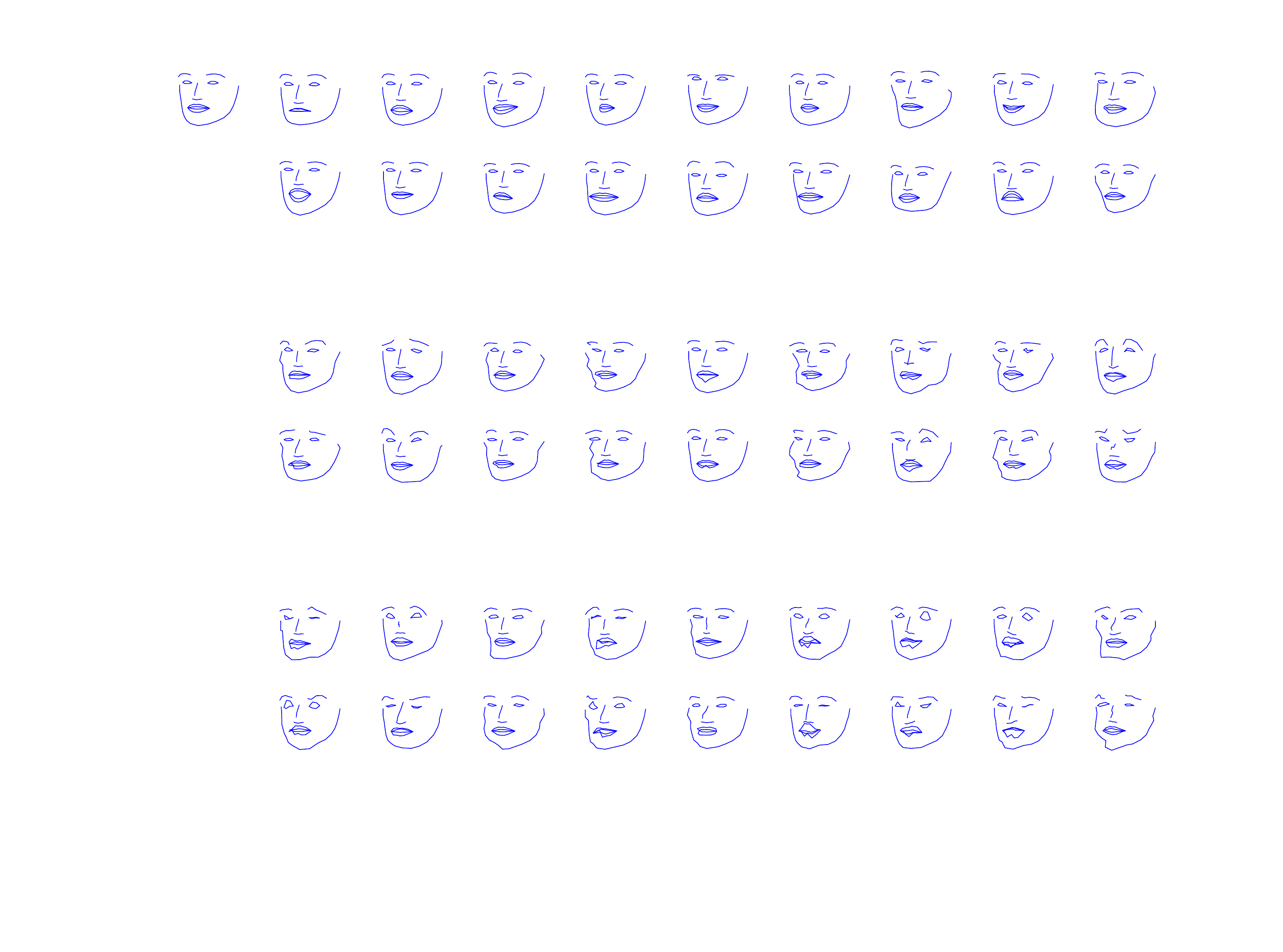}}
\subfigure[]{\includegraphics[width=0.47\textwidth, trim={4cm 3.8cm 5.5cm 0.5cm},clip]{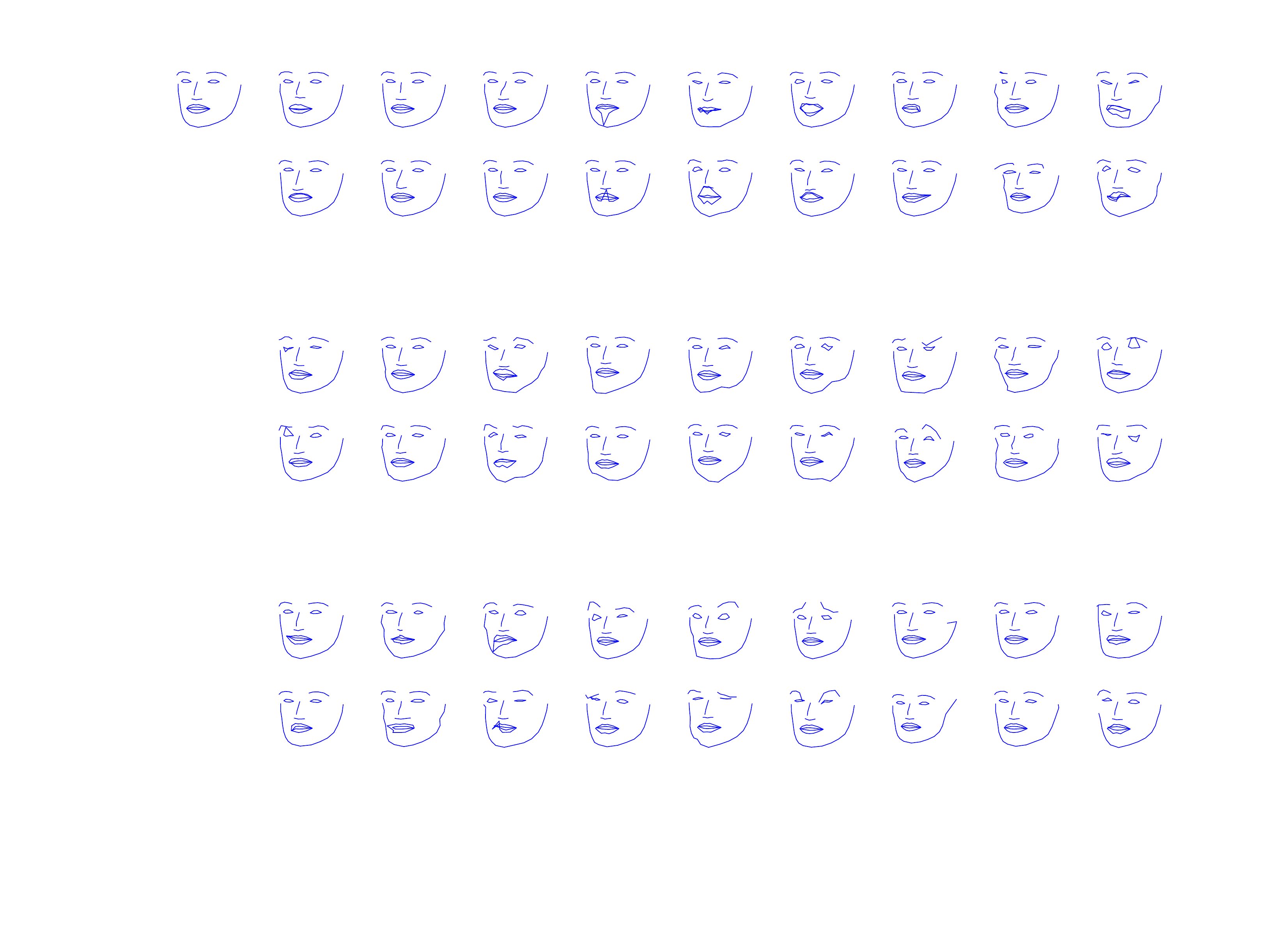}}\\
\subfigure[]{\includegraphics[width=0.15\textwidth]{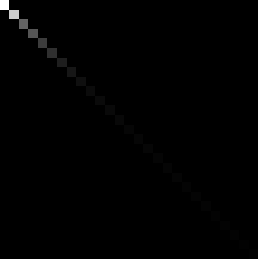}} \hspace{4cm} \subfigure[]{\includegraphics[width=0.15\textwidth]{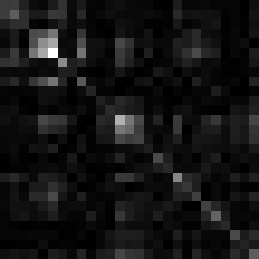}}\\
\end{center}
\caption{Rank-$1$ shape basis decomposition for LS3D-W data set with $K=27$. (a) The rigid affine 3D shape $\B_0$, and the $K$ estimated 3D Rank-$1$-PCA basis shapes $\B=\B_0 \pm \alpha_k \hat{\B}_k$, where $\alpha_k$ is a positive scalar. The components have been ordered with decreasing variance. (b) Correspondingly, the rigid affine 3D shape and the $K$ estimated 3D Rank-$1$-ICA basis shapes. (c) The $K\times K$ covariance matrix of the Rank-$1$-PCA basis shape coefficients that illustrates the fact that the shape coefficients are uncorrelated. (d) The covariance matrix of the Rank-$1$-ICA basis shape coefficients. The high off-diagonal covariance suggests that the two rank-$1$ shapes belong to a statistically dependent subspace. To facilitate inspection, the components have been permuted so that the covariance is concentrated around the main diagonal.}\label{fig:faceModes}
\end{figure*}

\subsection{Independent Component Analysis}
\label{sec:ICA}

The independent component analysis (ICA) is a general method for blind source separation that intends to decompose the underlying signals into statistically independent factors by using higher order statistics of multidimensional observations characterised by the random vector $\mathbf{Z}$. ICA can be defined as the minimisation of mutual information 
\begin{equation}
    I(\mathbf{Z})=\sum_{j} H(Z_j) - H(\mathbf{Z}),
\end{equation}
where $H$ refers to differential entropy and $\mathbf{Y} =\A_{\mathrm{ICA}} \mathbf{Z}$ to a random vector corresponding to the columns in $\B'$. If the vectors are mean centred and white, it implies that the 
mixing matrix $\A_{\mathrm{ICA}}=\G^\T$ will be an orthogonal matrix, hence, 
%
%
%
\begin{equation}
 {\B}_{\mathrm{ICA}} \equiv   
 \A_{\mathrm{ICA}}^\T {\B}' = \G {\B}'.
 \end{equation}
where
the rows of ${\B}_{\mathrm{ICA}}$ will be in as statistically independent as possible. Here, we compute the orthogonal, separation matrix $\G$ by the FastICA algorithm \cite{Hyvarinen97}.

\subsection{Recovery of Degenerate Basis Shapes} \label{sec:recovery}

Let $\bb_k$ denote the $k$th row in $\B_{\mathrm{PCA}} \equiv \B'$ or $\B_{\mathrm{ICA}}$, depending whether the rank-$1$ PCA or ICA model is selected, respectively. We are searching for the solution to the problem
\begin{equation}
    \min_{\bd,\alpha} \sum_i \| \Delta \W^i - \sum_k \alpha_k^i \B_k^i \|^2_\mathrm{Fro}, \label{eq:problem}
\end{equation}
subject to $\|\B_k^1\|_\mathrm{Fro}=1$, for all $k$, where $\B_k^i=\M^i \bd_k \bb_k^\mathrm{T}$ are the rank-$1$ operators referring to the degenerate basis shapes, and $\alpha_k^i$ are the corresponding basis coefficients that can be computed by orthogonally projecting the differential measurement matrix blocks $\Delta \W^i$ onto the rank-$1$ operators. We first note a useful property, stated as follows.
\begin{lemma}
  $\B_k^i \perp \B_{k'}^{i'}$ in the operator inner product, $k \neq k'$. \label{lem:op}
\end{lemma}

\begin{proof} 
We may write
\begin{equation}
\begin{split}
\langle \B^i_k, \B_{k'}^{i'} \rangle &= \langle \mathrm{vec}\{\B^i_k\}, \mathrm{vec}\{\B_{k'}^{i'}\} \rangle = 
\langle \M^i \bd_k,\M^{i'} \bd_{k'} \rangle \langle \bb_k, \bb_k' \rangle, 
\end{split}
\end{equation}
which vanishes for $k \neq k$' since $\bb_k \perp \bb_k'$.
\end{proof}

Now, we are ready to show how we minimise (\ref{eq:problem}). It is a non-linear minimisation over the unknowns $\bd_k$, which define the affine mapping of the rank-$1$ component $k$ onto the object coordinate frame, as stated by the following theorem. 

\begin{algorithm}[tb]
  \begin{enumerate}
      \item Form the translation corrected measurement matrix $\W$, as in (\ref{eq:factorisedform}).
      \item Decompose $\W$ into the rigid $\W_0$ and non-rigid $\Delta \W$ part as in (\ref{eq:fac.tomasi}) and (\ref{eq:nonrigid.part}), respectively. 
      \item Factorise the non-rigid part as $\Delta \W = \M' \B'$, where $\M' = \frac{1}{\sqrt{J}} \U \Sm$ and $\B' = \sqrt{J} \V^\T$.
      \item Do either
      \begin{enumerate}
        \item Compute the PCA basis by assuming $\G=\I$ and so that $\B_\mathrm{PCA}=\B'$; or
        \item Find the orthogonal transformation $\G$ and ICA basis by FastICA  \cite{Hyvarinen97} so that $\B_\mathrm{ICA}=\G \B'$.
      \end{enumerate}
      
      \item Find the component affine back-projections $\bd_k$, $k=1,2,\ldots,K$, by minimising (\ref{Eq:max.prob}). 
      \item Form the rank-$1$ basis shapes $\B_k^i=\M^i_0 \bd_k \bb_k^\T$, $i=1,2,\ldots,I$, where $\bb_k^\T$ is the $k$th row of $\B_\mathrm{PCA}$ or $\B_\mathrm{ICA}$, $k=1,2,\ldots,K$. 
      \item Solve the basis coefficients by orthogonal projection $\alpha_k^i=\langle \Delta \W^i, \B^i_k \rangle / \sqrt{\langle \B^i_k, \B^i_k \rangle}$, $i=1,2,\ldots,I$, $k=1,2,\ldots,K$.
  \end{enumerate}
  \caption{Non-rigid Structure From Motion by Rank-$1$ Basis Shapes} \label{alg:ours}
\end{algorithm}

\begin{theorem}
  The minimisation problem (\ref{eq:problem}) is equivalent to the set of 
  maximisation problems
  \begin{equation}
     \max_{\bd_k} \sum_i \frac{  \bd_k^\T (\bb_k \otimes \M^i)^\T  \w^i \w_i^\T (\bb_k \otimes \M^i) \bd_k }{  \bd_k^\T (\bb_k \otimes \M^i)^\T   (\bb_k \otimes \M^i) \bd_k }  
  \label{Eq:max.prob}
  \end{equation}
  subject to $\| (\bb_k \otimes \M^1) \bd_k \|=1$, for all $k \neq 0$, where 
the symbol $\otimes$ indicates the Kronecker product. 
\end{theorem}
\begin{proof} Let $\mathcal{W}^i$ be the operator subspace spanned by the operators $\B^i_k$, $k=1,2,\ldots,K,$ or, $\mathcal{W}^i=\mathrm{span}(\B^i_1,\B^i_2,\ldots,\B^i_K)$. 
%
%
  The problem (\ref{eq:problem}) is equivalent to the problem
    \begin{equation}
        \min_{\bd} \sum_i \| \proj_{{\mathcal{W}^i}^\perp}\{\W^i\} \|^2: \quad \|\B_k^1\|_\mathrm{Fro}=1 \ \forall k \label{eq:altproblem} 
    \end{equation}  
where $\proj_{\mathcal{W}^\perp}$ denotes the orthogonal projector onto the orthogonal complement of $\mathcal{W}$. Let $\mathcal{W}^i_k$ be the $1$-dimensional operator subspace spanned by the operator $\B^i_k$. 
Using the orthogonality of the operators in the form of Lemma~\ref{lem:op}, (\ref{eq:altproblem}) is equivalent to the set of problems, 
    \begin{equation}
    \begin{split}
        \min_{\bd_k}& \sum_i  \| \proj_{{\mathcal{W}^i_k}^\perp}\{\W^i\} \|^2 : \quad \|\bb_k^1\|=1 \\ &
        \Leftrightarrow \min_{\bd_k} \sum_i \left \| \left( \I - \frac{\bb^i_k {\bb^i_k}^\T}{\| \bb^i_k \|^2} \right)\w^i \right\|^2: \quad \|\bb_k^1\|=1
    \end{split} \label{eq:problemset}
    \end{equation}  
where $\bb^i_k = \mathrm{vec}(\B^i_k) = (\bb_k \otimes \M^i) \bd_k$ and $\w^i = \mathrm{vec}(\W^i)$. Using the the fact that the orthogonal projection is idempotent, (\ref{eq:problemset}) takes the form
\begin{equation}
\begin{split}
    &\min_{\bd_k}  \sum_i {\w^i}^\T \left( \I - \frac{\bb^i_k {\bb^i_k}^\T}{\| \bb^i_k \|^2} \right)\w^i : \quad \|\bb_k^1\|=1 \  \\
    &\Leftrightarrow \max_{\bd_k}  \sum_i {\w^i}^\T \left(  \frac{\bb^i_k {\bb^i_k}^\T}{\| \bb^i_k \|^2} \right)\w^i : \quad \|\bb_k^1\|=1 \\
               &\Leftrightarrow \max_{\bd_k}    \sum_i \frac{\bd_k^\T (\bb_k \otimes \M^i)^\T \w^i {\w^i}^\T (\bb_k \otimes \M^i)\bd_k}{\bd_k^\T(\bb_k \otimes \M^i)^\T (\bb_k \otimes \M^i)\bd_k}  : 
                \hfill \quad \|(\bb_k \otimes \M^1) \bd_k\|=1.
\end{split}
\end{equation}
\end{proof}

Our method is now complete. It is summarised in Algorithm~\ref{alg:ours}.

%

\section{Experiments}
\label{sec:experiments}

We evaluated both variants of the proposed method with several data sets. As reference scores, we used those reported in \cite{Brandt18}. The reference methods are Dai \etal's Pseudoinverse (PI) and Block Matrix Method (BMM) \cite{Dai12}, Kong and Lucey's Priorless decomposition \cite{Kong16}, and Brandt \etal's ISA decomposition \cite{Brandt18}.

\subsection{Shark}

Torressani's synthetic shark \cite{Torresani08} data set is a classic test case. It is a degenerate dataset $(I=240,\ J=91)$ with its rank equal to five after the translation correction. Due to the degeneracy, the 3D reconstruction is not unique but it has a three-parameter-family of solutions even when one uses only one rank-$3$ deformation basis shape. Our method, since built upon the assumption of rank-$1$ basis shapes, is able to exactly match the degree of freedoms of the data by setting $K=2$\footnote{As a proof of the principle, we additionally computed the case $K=3$ corresponding to the experiments \cite{Brandt18} to see that there was no degradation in our result.}. Since our reconstruction is affine, we use the same evaluation metric as in \cite{Brandt18}, i.e., the relative reprojection error or the inverse signal to noise ratio on the image plane. The results are shown in Table~\ref{tab:results}. It is also instructive to see what the $1$-dimensional deformation modes represent. It can be seen (Fig.~\ref{fig:shark}) that though the Rank-$1$-PCA modes are clearly distinctive, the Rank-$1$-ICA modes are much more intuitive as they correspond to mid body movement and the movement of the head and they could be well seen as statistically independent movements.

\begin{table*}[b]
\caption{Comparison of the NRSFM methods measured by the relative reprojection error. 
} \label{tab:results}
\centering
\begin{tabular}{|l|c|c|c|c|c|c|}
\hline
Inverse SNR $[\%]$ & PI \cite{Dai12} & BMM \cite{Dai12} & Priorless  \cite{Kong16} & ISA \cite{Brandt18} & R1-PCA & R1-ICA\\
\hline
Shark & $3.5^\dagger$ & $0.33^\dagger$ & $160^\dagger$ & ${ \bf 0.12}^\dagger$ & ${ \bf 0.12}$ & ${ \bf 0.12}$\\
Balloon & $0.11^\dagger$  & ${\bf 0.012}^\dagger$ & $1.2^\dagger$ & $0.12^\dagger$ & $0.046$ & $0.072$\\
Face LS3D-W & $0.025^\dagger$ & $0.024^\dagger$ & $0.93^\dagger$ & ${0.014}^\dagger$ & 0.011 & $\bf{0.0096}$ \\
\hline
\end{tabular}\\
{\quad \quad \quad \quad \footnotesize $\mbox{}^\dagger$ The score adopted from \cite{Brandt18}.}
\vspace{-4mm}
\end{table*}

\begin{figure}
    \centering
    \subfigure[]{
    \includegraphics[width=0.45\textwidth,trim={1cm 0.6cm 1cm 1cm},clip]{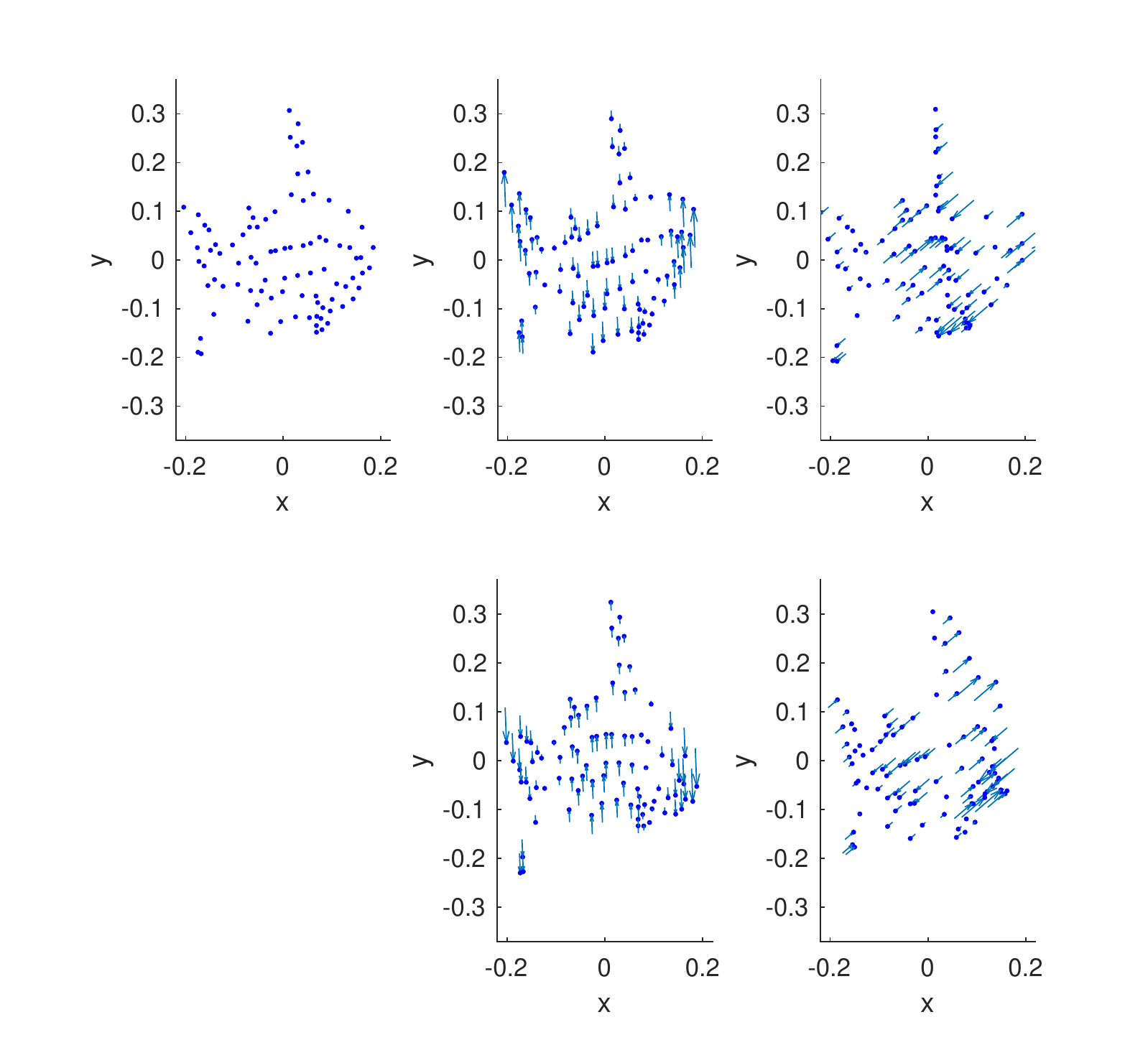}}
    \subfigure[]{
    \includegraphics[width=0.45\textwidth,trim={1cm 0.6cm 1cm 1cm},clip]{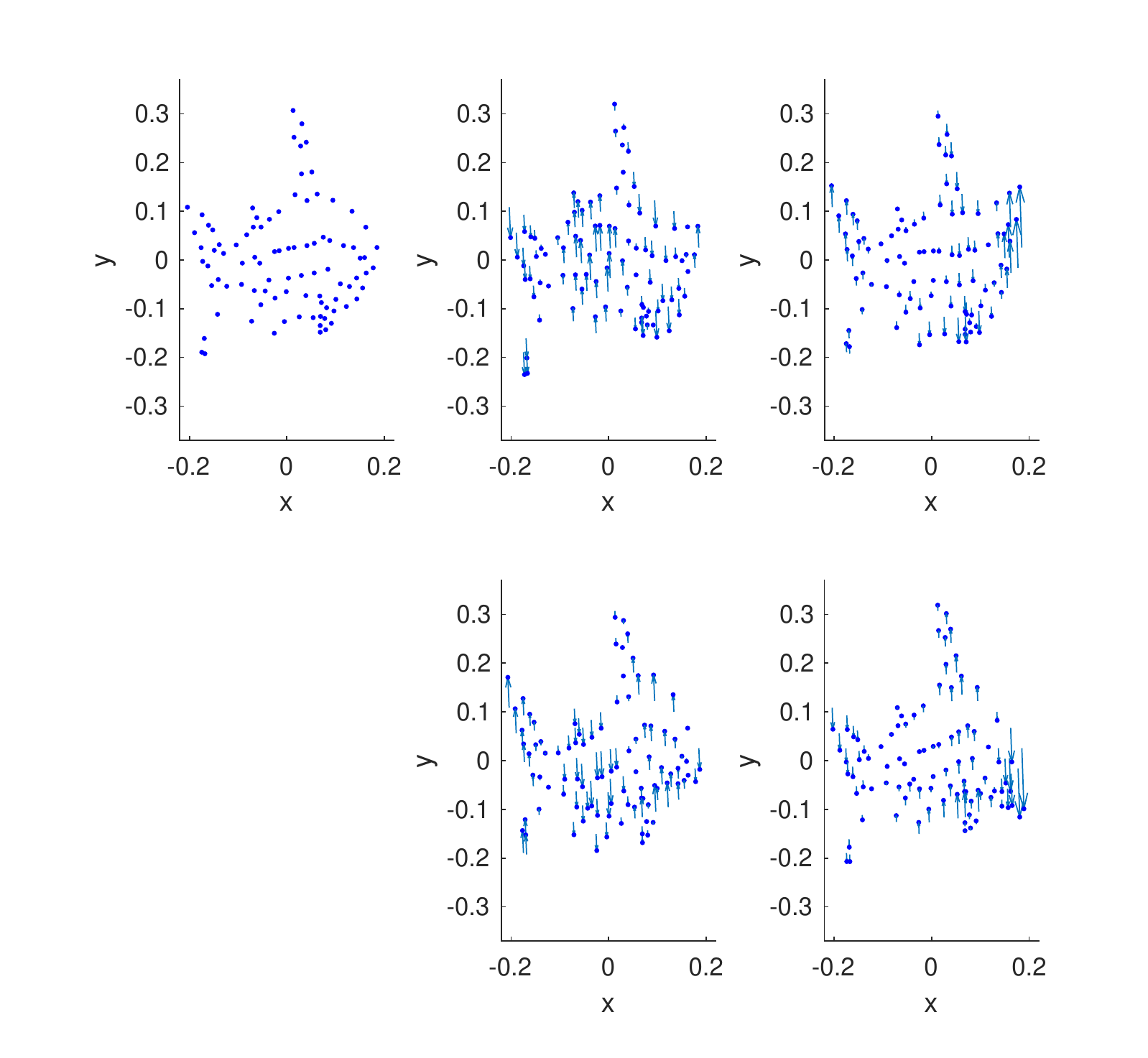}}
    \caption{Torressani's Shark sequence decomposed with $K=2$. (a) Rank-$1$-PCA; from left to right: mean, rigid 3D shape, first deformation mode, second deformation mode; the modes can be interpreted as the body bending and diagonal compression--stretching, respectively. (b) The corresponding Rank-$1$-ICA, where the two deformation modes can be interpreted as the middle body and front body movements with respect to the mean, respectively.}
    \label{fig:shark}
\end{figure}

\subsection{Balloon}

The Balloon dataset is a benchmarking dataset from the NRSFM challenge \cite{Jensen18}, created by simulating perspective reprojections $(I=51)$ of real tracked 3D points ($J=211$). The virtual camera made a circular motion sequence while the 3D ground truth of the dataset is not publicly available. Our method is affine, hence, there is an unknown affine transformation between the affine reconstruction result and the Euclidean 3D structure. To qualitatively show the nature of the deformation modes, we first computed the factorisation on $K=7$ $1$-dimensional deformation modes 
(Fig.~\ref{fig:balloon}). For quantitative evaluation, and to match our result with those reported in \cite{Brandt18}, we set $K=15$ and evaluated the relative reprojection error. The results are in Table~\ref{tab:results}. It can be seen that Rank-$1$-PCA gave the second best better result, slightly better than the Rank-$1$-ICA but the Block Matrix Method (BMM) by Dai \etal \cite{Dai12} performed the best, though it is computationally much more demanding.

\begin{figure}[!tb]
\centering
    \subfigure[]{
    \includegraphics[width=\textwidth,trim={4cm 1.2cm 3cm 1.5cm},clip]{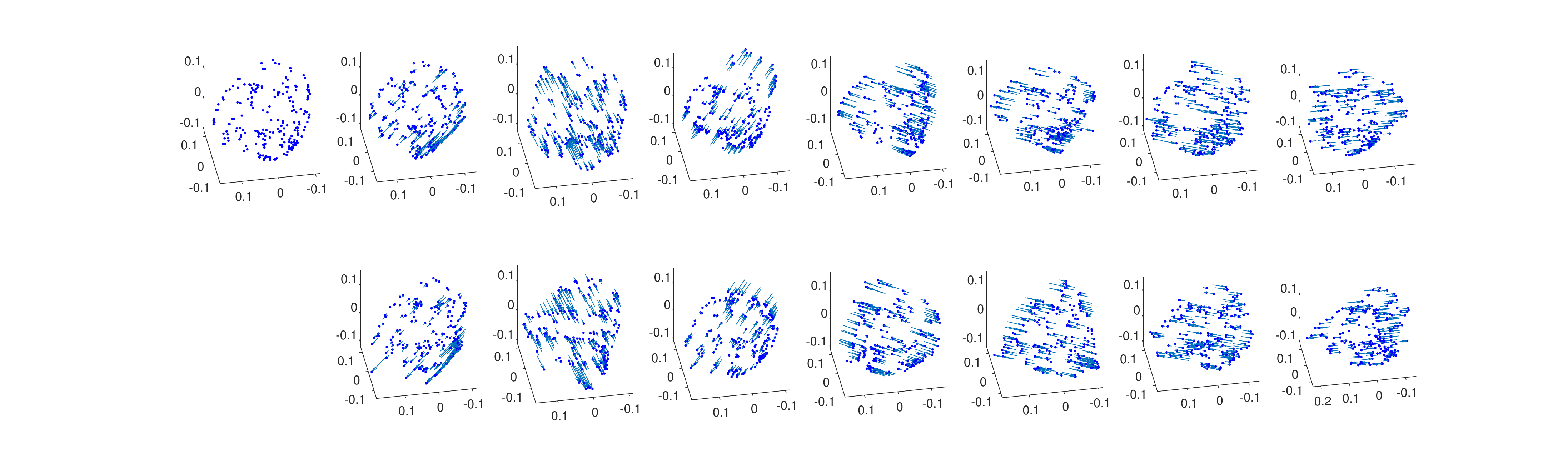}}
    \subfigure[]{
    \includegraphics[width=\textwidth,trim={4cm 1.2cm 3cm 1.5cm},clip]{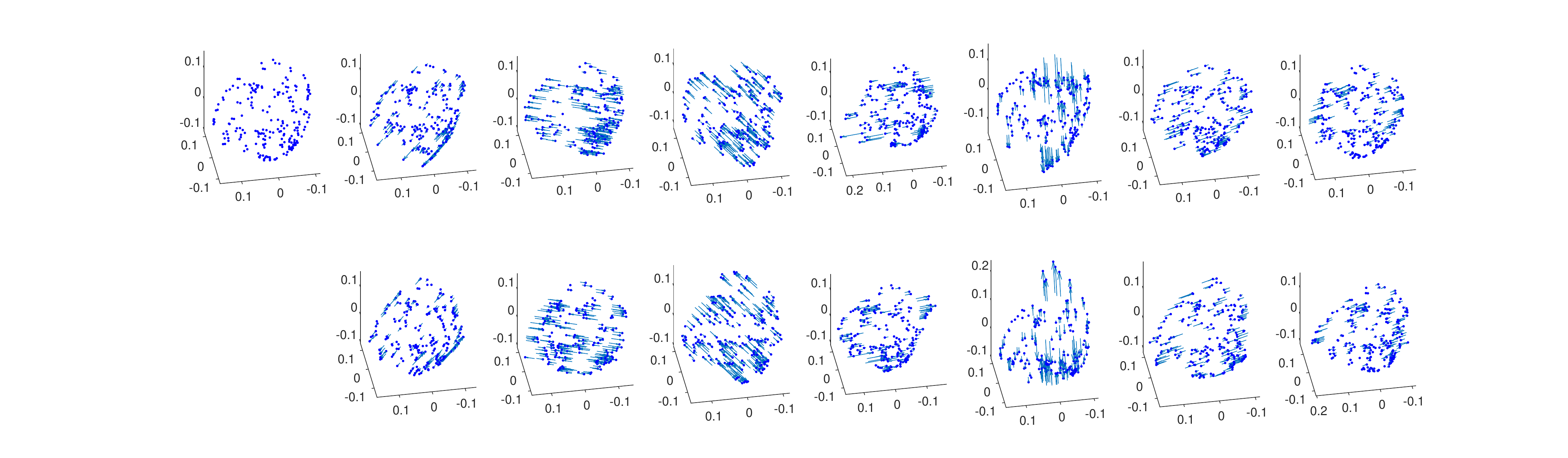}}
    \vspace{-4mm}
    \caption{Balloon deflation decomposition (K=7) onto the mean and (a) Rank-$1$-PCA shapes or (b) Rank-$1$-ICA shapes.} \label{fig:balloon}
%
    \includegraphics[width=\textwidth,trim={4cm 0.5cm 3cm 0cm},clip]{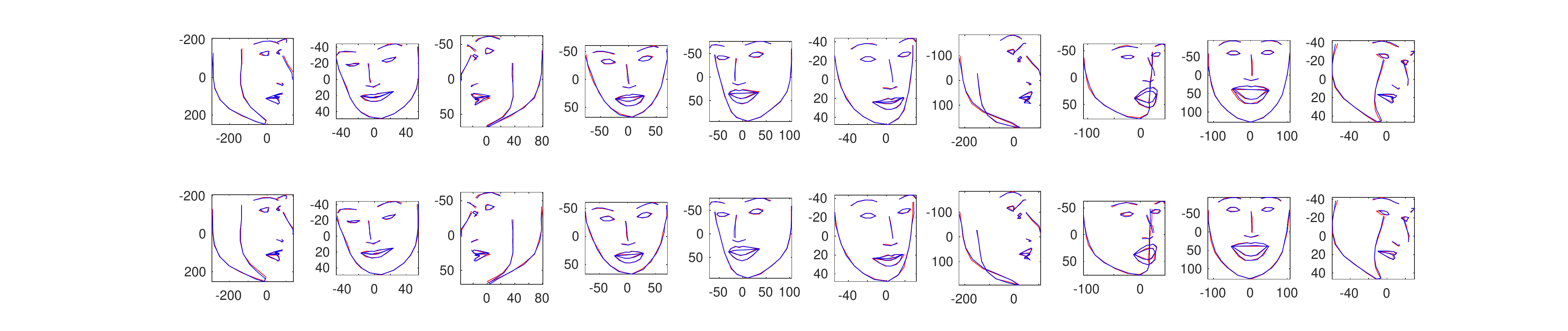}
    \vspace{-5mm}
    \caption{Reprojections onto random faces on the LS3D-W dataset. (top) Rank-$1$-PCA, (bottom) Rank-$1$-ICA (c.f. \cite{Brandt18}). Blue denotes the ground truth, red the reprojection, respectively. Both variants yield an accurate representation on arbitrary head pose and expression.} \label{fig:faceRec}
\end{figure}
 
\subsection{LS3D-W}

Finally, we experimented the LS3D-W dataset \cite{Bulat2017} containing $I=7200$ human face projections, of which $J=68$ 2D matched points were extracted. The data set is challenging due to the fact that the projections are in random orientations. 
The qualitative comparison between the Rank-$1$-PCA and Rank-$1$-ICA is shown in Fig~\ref{fig:faceModes}. 
%
The covariance matrix of Rank-$1$-ICA coefficients reveals the dependent subspaces whereas Rank-$1$-PCA coefficients are by construction uncorrelated. The former is preferred when one intends to investigate 
independent subspaces such as those of lip movements and face size changes. 
The quantitative evaluation (Tab.~\ref{tab:results}, Fig~\ref{fig:faceRec}) shows that the proposed Rank-$1$-ICA method yields the best result and Rank-$1$-PCA the second. 
The computation of the result by the Dai's method takes about 2 CPU days, Kong's and Luceys method about 6 CPU hours, and Brandt's ISA about twenty CPU minutes \cite{Brandt18}. Our methods are able to yield the result in a few minutes thus being the fastest. 

\section{Conclusions} \label{sec:conclusions}

In this paper, we have shown how the non-rigid structure-from-motion problem can be solved by using rank-$1$ shapes. It yields a natural interpretation as a deformation mode is defined by a right singular vector of the classic low-rank factorisation model. This singular vector is back-projected into 3D space into a certain direction, which is solved as part of the factorisation. We proposed two variants of the methods referred to as Rank-$1$-PCA and Rank-$1$-ICA, of which the latter is able to additionally reveal the statistically dependent subspaces among the deformations. Moreover, in contrast to earlier methods, there is no need to enforce the block structure of the motion matrix as it is not required that the right singular vectors are grouped into subgroups of three. This as a clear advantage over the previous formulations, since the realistic structure may contain nested statistical dependencies over arbitrary dimensions that are otherwise difficult to model.



{\small
\bibliography{egbib}
}

\end{document}